\documentclass{article}

\usepackage{wustyle}
\usepackage{fullpage}
\usepackage{wustyle}

\usepackage{braket}
\usepackage[top=0.9in,bottom=1.1in,left=1.1in,right=1.1in]{geometry}
\usepackage{authblk}

\title{The $L^\infty$ Learnability of Reproducing Kernel Hilbert Spaces}

\author{ 
Hongrui Chen\footnote{Peking University, \texttt{hongrui\_chen@pku.edu.cn}}
~~~~~~
Jihao Long\footnote{Princeton University, \texttt{jihaol@princeton.edu}}
~~~~~~
Lei Wu\footnote{Peking University, \texttt{leiwu@math.pku.edu.cn}}
}
\begin{document}

\maketitle 
\begin{abstract}
In this work, we analyze the learnability of reproducing kernel Hilbert spaces (RKHS) under the $L^\infty$ norm, which is critical for understanding the performance of kernel methods and random feature models in safety- and security-critical applications. Specifically, we relate the $L^\infty$ learnability of a RKHS to the spectrum decay of the associate kernel and both lower bounds and upper bounds of the sample complexity are established. In particular, for dot-product kernels on the sphere, we identify conditions when the $L^\infty$ learning can be achieved with polynomial samples.
Let $d$ denote the input dimension and assume the kernel spectrum roughly decays as $\lambda_k\sim k^{-1-\beta}$ with $\beta>0$. We prove that  if $\beta$ is independent of the input dimension $d$,  then functions in the  RKHS   can be learned efficiently under the $L^\infty$ norm, i.e.,  the sample complexity depends polynomially on $d$. In contrast, if $\beta=1/\mathrm{poly}(d)$, then the $L^\infty$ learning requires exponentially many samples.
\end{abstract}

\section{Introduction}

 Traditional machine learning have primarily focused on average-case scenarios, seeking to optimize performance measures that capture the average behavior. For instance, in regression problem, $L^2$ error is the most popular 
metric. 
 However, in many safety- and security-critical applications, it is also crucial to consider worst-case scenarios, where the performance of the learning algorithm needs to be guaranteed regardless of the data distribution. For instance, the worst-case guarantee is required in ensuring the adversarial robustness \cite{szegedy2013intriguing,goodfellow2014explaining}, solving PDEs \cite{han2018solving,raissi2019physics}, and  understanding reinforcement learning \cite{long2022perturbational}. 
However,  this needs us to move beyond the $L^2$ metric and consider  the performance of machine learning algorithms in $L^\infty$ metric.

In this paper, our goal is to 
study the $L^\infty$ learning of functions in reproducing kernel Hilbert space (RKHS)  \cite{aronszajn1950theory}. In algorithm designing, RKHS have emerged as a powerful framework, providing a flexible and expressive class of models that can capture complex patterns in data~\cite{gretton2012kernel}. In theoretical analysis, 
because of the reproducing property, RKHS has been widely adopted in analyzing the performance of kernel-based methods, random feature models~\cite{Rahimi2007RandomFF,bach2017equivalence}, and neural networks \cite{jacot2018neural,weinan2020comparative,arora2019exact}. However, existing  analyses of learning RKHS mostly focus on the $L^2$ metric. Understanding the $L^\infty$ learnability  of RKHS  is still largely open.

As particular examples, we will mostly focus on  dot-product kernels on the unit sphere \cite{smola2001regularization}. A kernel $k: \SS^{d-1}\times\SS^{d-1}\mapsto\RR$ is said to be dot-product if there exist $\kappa: [-1,1]\mapsto\RR$ such that 
\begin{equation}\label{eqn: dot-product}
  k(x,x') = \kappa(x^Tx').
\end{equation}
Dot-product kernels are one of the most popular kernels used in kernel-based methods\cite{smola2001regularization,scetbon2021spectral,cho2009kernel,bach2017equivalence}. In particular, the neural tangent kernels \cite{cho2009kernel,jacot2018neural} arising from understanding neural networks all are dot-product \cite{bietti2021deep,daniely2017sgd,arora2019exact}.

{\bf Our contribution.}~
In this paper, we provide a spectral-based analysis of  the $L^\infty$ learnability of functions in RKHS. Specifically, we establish both  lower bounds and upper bounds (only for dot-product kernels) of the sample complexity by using the eigenvalue decay of the associate kernel. Let $\{\mu_j\}_{j\geq 1}$ be the  eigenvalues in decreasing order, $\Lambda(m) = \sum_{j=m+1}^{+\infty} \mu_j$, and $n$ the number of samples.  Our results 
can be summarized as follows. 
\begin{itemize}
  \item We first establish an upper bound of the $L^\infty$ excess risk for dot-product kernels by using the spectrum decay: $\Lambda(n)$. In particular, if there exists an absolute constant $\beta>0$ such that the eigenvalue decays like $\mu_j\propto j^{-1-\beta}$, then $L^\infty$ error is given by $O(d^{\frac{\beta}{2\beta+1}}n^{-\frac{\beta}{2(2\beta+1)}})$, where the dependence on $d$ is only polynomial. This implies that achieving $L^\infty$ learning is feasible as long as the eigenvalues decay reasonably fast.
 It is important to note that the learning algorithm we employed is the standard kernel ridge regression (KRR).

\item We next establish a lower bound on the minimax $L^\infty$ estimation error. In particular, the lower bound is proportional to the quantity $\Lambda(n)$, indicating that the decay rate of the $L^\infty$ excess risk follows the form $O(n^{-2\beta_d})$ when the eigenvalues satisfy the condition $\mu_j\sim j^{-1-\beta_d}$. Consequently, the $L^\infty$ learning suffers from the curse of dimensionality if $\beta_d=1/\poly(d)$.

\item We also apply our bounds to RKHSs associated with dot-product kernels of the form: $k(x,x')=\EE_{v\sim \tau_{d-1}}[\sigma(v^Tx)\sigma(v^Tx')]$, where $\tau_{d-1}$ is the uniform distribution over $\SS^{d-1}$ and $\sigma(\cdot)$ is an activation function. Specifically, when 
$\sigma(\cdot)$ is a smooth function (such as sigmoid, softplus, and SiLU), $L^\infty$ learning in the corresponding RKHSs   needs only polynomially many samples. In contrast, when $\sigma(\cdot)$ is nonsmooth (e.g., ReLU), $L^\infty$ learning inevitably suffers from the curse of dimensionality. These conclusions follow from  analyzing how the eigenvalue decay of $k(\cdot,\cdot)$ depends on the smoothness of $\sigma(\cdot)$.

\end{itemize}
% In a summary, the $L^\infty$ learnability of RKHS 
% associated with the dot-product kernel can be comprehensively captured by the kernel spectrum decay $\Lambda(n)$, despite the rapid growth of the $L^\infty$ norm of the sphere harmonics. This revelation is realized by establishing a dual relationship between the $L^\infty$ learnability of RKHS and the approximability of the Barron space through the random feature model, which has been carefully studied in \cite{Wu2021ASA}. 

\subsection{Related Work}
% <<<<<<< HEAD
% The first study of $L^\infty$ learnability in RKHS was conducted in the noiseless scenario as outlined in \cite{Kuo2008MultivariateLA}, establishing both the upper and lower bounds in relation to kernel spectrum decays: $\Lambda(n) = \sum_{i=n+1}^{+\infty} \lambda_i$. The lower bound, specifically, is quite lucid: it posits that the worst-case error of any $L^\infty$ learning algorithm within an RKHS can be bounded below by the kernel spectrum decays (see Theorem \ref{lower bound} for a detailed statement). Subsequent research \citep{long20212,long2023reinforcement} has predominantly employed the same methodology to investigate the lower bound. This research has identified precise instances of RKHS where $L^\infty$ learning is detrimentally affected by the curse of dimensionality, accomplished by estimating the decay of the kernel spectrum. Compared to previous works, where the input data are fixed and output data are accurate, we consider the case where the input data are i.i.d. sampled from a given distribution and output data are noisy, which is closer to practice.
% =======
To the best of our knowledge, the first study of $L^\infty$ learnability in RKHS was conducted by \cite{Kuo2008MultivariateLA}, establishing both the upper and lower bounds in relation to the kernel spectrum decays: $\Lambda(n) = \sum_{i=n+1}^{+\infty} \mu_i$. The lower bound, specifically, shows that the worst-case error of any $L^\infty$ learning algorithm within an RKHS can be bounded below by the kernel spectrum decays. Subsequent research \cite{long20212,long2023reinforcement} has predominantly employed the same methodology to investigate the lower bound. Note that \cite{Kuo2008MultivariateLA} considers a setup where the training samples $\{x_i\}_{i=1}^n$ are deterministic. In their framework,  the hard functions are data-dependent and consequently, for different training samples $\{x_i\}_{i=1}^n$, the hard target functions  might be different. In contrast, we consider the standard statistical learning setup, where our training data $\{x_i\}_{i=1}^n$ are \iid samples drawn from a specific distribution. Our lower bound is established for the standard minimax errors \cite{wainwright2019high}, where the hard functions are independent of training data.

% In addition, we identify precise instances of RKHS where $L^\infty$ learning is detrimentally affected by the curse of dimensionality, accomplished by estimating the decay of the kernel spectrum.

Let $e_i$ be the eigenfunctions associated with the eigenvalue $\lambda_i$ for $i\in \NN$. The upper bound in \cite{Kuo2008MultivariateLA} requires an uniform $L^\infty$ bound for the eigenfunctions: $\sup_{i \in \mathbb{N}^+} \|e_i\|_\infty < +\infty$. This assumption is also found in \cite{caponnetto2007optimal,Pozharska2022ANO,long2022perturbational} for the upper bound. Moreover, in \cite{mendelson2010regularization}, this condition is relaxed by requiring $\sup_{i \in \mathbb{N}^+} \mu_i^\epsilon \|e_i\|_\infty < +\infty$ where $\epsilon \in [0,1/2)$ is a universal constant. The upper bound in this case correlates with the modified eigenvalue decay: $\Lambda^\epsilon(n) = \sum_{i=n+1}^{+\infty} \mu_i^{1-2\epsilon}$. However, these assumptions exclude dot-product kernels on the sphere since the $L^\infty$ norm of sphere harmonics, the eigenfunctions of the dot-product kernels, grows rapidly \cite{burq2014probabilistic}. %as evidenced by the following proposition:
%\begin{proposition}
%    Let $\{e_i\}_{i=1}^{+\infty}$ be an orthonormal basis under the uniform distribution on $\mathbb{S}^{d-1}$, which is formed by the sphere hamornics. Let $\{\lambda_i\}_{i=1}^{+\infty}$ be a nonincreasing positive sequence. If there exists a constant $\epsilon\in [0,\frac{1}{2})$ such that
%    \begin{equation}
%        \sup_{i \in \mathbb{N}^+}\lambda_i^\epsilon \|e_i\|_\infty < +\infty,
%    \end{equation}
%    then,
%    \begin{equation}
%        \sup_{i \in \mathbb{N}^+}\lambda_i i^{\frac{d-2}{2\epsilon}} < +\infty.
%    \end{equation}
%\end{proposition}
Therefore, existing results for the upper bound are only applicable to the dot-product kernel with extremely fast spectrum decay when $d$ is large.

Concurrent work \cite{dong2023linftyrecovery} also explores $L^\infty$ learnability for  dot-product kernels. However, they focus on Gaussian random field and their results are not applicable to an entire RKHS. In this work, we establish the upper bound of $L^\infty$ learnability of the entire RKHS with dot-product kernel based on the kernel spectrum decay, despite the rapid growth of the $L^\infty$ norm of sphere harmonics. Furthermore, it is important to note that the $L^\infty$ learnability in \cite{dong2023linftyrecovery} is established for an add-hoc algorithm, whereas ours is for the standard KRR algorithm.

% Another important difference from \cite{dong2023linftyrecovery} is that our upper bound holds for a standard estimator: the kernel ridge regression

% On the other hand, $L^\infty$ estimation within RKHS corresponding to the dot-product kernel associated with non-smooth activation functions such as RELU remains susceptible to the curse of dimensionality.

\section{Preliminaries}
 \paragraph*{Notations.} 
 For $\cX \subset \mathbb{R}^d$, we denote by $\cP(\cX)$ the set of probability measures on $\cX$ and $\cM(\cX)$ the space of signed Radon measures equipped with the total variation norm $\|\mu\|_{\cM(\cX)}=\|\mu\|_{\mathrm{TV}}$.  For a probability measure $\gamma \in \cP(\cX)$,   we use $\langle \cdot,\cdot \rangle_\gamma$ and $\|\cdot\|_\gamma$ to denote the $L^2(\gamma)$ inner product and norm, respectively. For any vector $v$ in Euclidean space, denote by $\|v\| = \left(\sum_{i} |v_i|^2 \right)^{1/2}$ the Euclidean norm.  Let $\SS^{d-1}$ be the unit sphere on $\RR^d$ and $\tau_{d-1}$ denote the uniform measure on $\SS^{d-1}$.

 We use $a \lesssim b$ to mean $a \leq Cb$ for an absolute constant $C > 0$ and $a \gtrsim b$ is defined analogously. We use $a \sim b$ if there exist absolute constants $C_1, C_2 > 0$ such that $C_1b \leq a \leq C_2b$. We also use $a\lesssim_{\gamma} b$ to denote that $a\leq C_\gamma b$ for a constant $C_\gamma$ that depends only on $\gamma$ and $a\gtrsim_\gamma b$ is defined analogously. 

 \paragraph*{Mercer decomposition and RKHSs.}
We recall some facts about the eigen decomposition of a kernel.
For any kernel $k:\cX\times\cX\mapsto\RR$ and a probability measure $\gamma \in \cP(\cX)$, the associated  integral operator $ \cT_k^\gamma: L^2(\gamma)\mapsto L^2(\gamma)$  is given by 
$
    \cT_k^\gamma f = \int k(\cdot,x) f(x) \dd \gamma(x).
$
When $k$ is continuous and $\cX$ is compact, the Mercer's theorem guarantees the existance of eigen decomposition of $k$:
\begin{equation}
k(x,x')=\sum_{i=1}^\infty \mu_i e_i(x)e_i(x'). 
\end{equation}
Here $\{\mu_i\}_{i=1}^\infty$ are the eigenvalues in a decreasing order and $\{e_i\}_{i=1}^\infty$ are the orthonormal eigenfunctions satisfying $\int e_i(x)e_j(x)\dd\gamma(x)=\delta_{i,j}$. The trace of $k$ satisfies $\sum_{i=1}^\infty \mu_i = \int_\cX k(x,x) \d \gamma(x)  $ 
Note that the decomposition depends on the input distribution $\gamma$ and when needed, we will denote by $(\mu_i^{k,\gamma},e_i^{k,\gamma})$ the $i$-th eigenvalue and eigenfunction to explicitly emphasize the influence of $k$ and $\gamma$.  By using the Mercer decomposition, the RKHS can  be defined as
$
    \cH_k=\{f: \|f\|_{\cH_k}<\infty\},
$ 
where the RKHS norm is given by 
\begin{equation}
  \|f\|_{\cH_k}^2 = \sum_{i=1}^\infty \frac{\langle f,e_i\rangle^2_\gamma}{\mu_i}.
\end{equation}
% In particular, we are interested in the dot-product kernel \eqref{eqn: dot-product}, whose eigen decomposition  can explicitly written by using the spheretical harmonics.

\paragraph*{Legendre polynomials and spherical harmonics.}
Legendre polynomials and spherical harmonics are critical for the analysis of dot-product kernel on the sphere. The Legendre polynomials in $d$-dimension is recursively defined by:
\begin{align*}
 & P_{0,d}(t)=0,\, P_{1,d}(t)=t,\\
& P_{k,d}(t)=\frac{2 k+d-4}{k+d-3} t P_{k-1,d}(t)-\frac{k-1}{k+d-3} P_{k-2,d}(t),\, k \geq 2 .
\end{align*}
Note that $\{P_{k,d}\}_{k=0}^\infty$ forms an complete orthogonal basis of $L^2(\tilde{\tau}_{d-1})$, where $\tilde{\tau}_{d-1}$ is the marginal distribution of the uniform distribution on the sphere, i.e., the distribution of $x_1$ for $x \sim \tau_{d-1}$. 

Let $\mathcal{Y}_k^d$ be the space of all homogeneous harmonic polynomials of degree $k$ in $d$ dimensions restricted on $\mathbb{S}^{d-1}$; the dimension of the space $\mathcal{Y}_k^d$ is $N(d, k) :=\frac{2k+d-2}{k} \tbinom{k+d-3}{d-2}$. Let $\left\{Y_{k, j}\right\}_{1 \leq j \leq N(d, k)}$ be an orthonormal basis of $\mathcal{Y}_k^d$ in $L^2\left(\tau_{d-1}\right)$. Then $Y_{k, j}: \mathbb{S}^{d-1} \mapsto \mathbb{R}$ is the $j$-th spherical harmonics of degree $k$ and $\left\{Y_{k, j}\right\}_{k \in \mathbb{N}, 1 \leq j \leq N(d, k)}$ forms an orthonormal basis of $L^2\left(\tau_{d-1}\right)$. Moreover, the degree-$k$ harmonics $\{Y_{k,j}\}_{1\leq j \leq N(d,k)} $ satisfies
\begin{align}\label{eqn: dot-2}
         \frac{1}{N(d,K)}\sum_{j=1}^{N(d,k)} Y_{k,j}(x)Y_{k,j}(x') = P_{k,d}(x^\top y),\quad \forall x,x' \in \SS^{d-1}.
\end{align}

\paragraph*{Dot-product kernels on the unit sphere.}
Sphere harmonics are the common eigenfunctions for all dot-product kernels on the sphere. Specifically, consider a dot-product kernel $k(x,x') = \kappa(x^\top x')$ on $\SS^{d-1}$. Then, the spectral decomposition of $k$ corresponding to the uniform measure $\tau_{d-1}$ is
$$   \kappa(x^\top x')= \sum_{k=0}^\infty \sum_{j=1}^{N(d,k)}\lambda_k Y_{k,j}(x) Y_{k,j}(x'),  $$
where $\{\lambda_k\}_{k\geq 0} $ are the eigenvalues of the kernel operator $\cT_k^{\tau_{d-1}}$ counted without multiplicity in decreasing order. Note that $N(d,k)$ is the mutiplicity  of the $k$-th eigenvalue $\lambda_k$.  It is important to note that $\{\mu_j\}_{j\geq 1}$ are the corresponding eigenvalues with multiplicity counted.

We refer to \cite{schoenberg1988positive} and \cite{atkinson2012spherical} for more details about the harmonic analysis on $\SS^{d-1}$.

\subsection{Random Feature Kernels}
We are in particular interested in the kernels with an integral representation because of the connection with random feature models (RFMs). Let $\cV$ be weight space equipped with a weight distribution $\pi$ and $\phi: \cX \times \cV \to \RR$ be a feature function.
Consider the RFM:
\begin{align}
    f_m(x;\theta) = \sum_{j=1}^m  c_j\phi(x,v_j),  \label{RFmodel}
\end{align}
where $\theta = \{c_j,v_j \}_{j=1}^m $, the inner weights $v_1,\cdots,v_m \in \cV$ are i.i.d. sampled from a weight distribution $\pi\in \cP(\cV)$. The output weights $ c_1,\cdots,c_m \in \RR^m$ are the learnable parameters. Given the random weight $v_1,\cdots,v_m$, functions of the form \eqref{RFmodel} induces an empirical kernel $\hat{k}: \cX \times \cX \to \RR$:
$
     \hat{k}(x,x') = \frac{1}{m}\sum_{j=1}^m \phi(x,v_j)\phi(x',v_j).
$
As the number of features $m$ tends to infinity, the empirical kernel converges to
\begin{align}\label{RFM}
          k(x,x') = \int_\cV \phi(x,v)\phi(x',v)\d \pi(v).
\end{align}
For this kernel, the corresponding RKHS norm $\cH_k$ admits following representation:
  \begin{align}\label{eqn: RKHS-rf}
\|f\|_{\cH_k}&:=\inf_{f=\int_{\cV} a(v)\phi(\cdot,v) \d \pi(v)} \|a\|_{\pi}.
    \end{align}
% Therefore, random feature models can be naturally taken as the monte-carlo approximations of functions in $\cH_k$. 
We refer to the kernel \eqref{RFM} as the \emph{random feature kernel} associate with the feature function $\phi$. 

Conversely, the following proposition establishes that any kernel can be expressed as a random feature kernel associate with some feature function:
\begin{proposition}\label{integral}
For any $k: \cX \times \cX \to \RR$ and any $\pi \in \cP(\cX)$, there exists a symmetric feature function $\phi: \cX \times \cX \to \RR$ such that $k(\cdot,\cdot)$ takes the form of random feature kernel \eqref{RFM}. 
Moreover, for dot-product kernel $k(x,x') = \kappa(x^\top x')$, if $\cX = \SS^{d-1}$ and $\pi = \tau_{d-1}$, there exists an activation function $\sigma: \RR \to \RR$ such that $\phi(x,v) = \sigma(v^\top x)$ satisfies \eqref{RFM}.
\end{proposition}
\begin{proof}
Consider the spectral decomposition of $k$ corresponding to the distribution $\pi$:
$$
k(x,x') = \sum_{i=1}^\infty \mu_ie_i(x)e_i(x'). 
$$
Choosing $\phi(x,v) = \sum_{i=1}^\infty \sqrt{\mu_i}e_i(x)e_i(v)$, we have
\begin{align*}
\int_\cX \phi(x,v)\phi(x',v) \d \pi(v) & = \sum_{i=1}^\infty \sum_{j=1}^\infty \sqrt{\mu_i\mu_j}\int_{\cX}e_i(x)e_i(v)e_j(x')e_j(v) \d \pi(v) \\
& = \sum_{i=1}^\infty \mu_ie_i(x)e_i(x') = k(x,x').
\end{align*}
Thus we complete the proof of the first part. Similarly, for dot product kernel $k(x,x') = \kappa(x^\top x')$, consider the spectral decomposition of $\kappa$:
$$
 \kappa(x^\top x')= \sum_{k=0}^\infty \sum_{j=1}^{N(d,k)}\lambda_k Y_{k,j}(x) Y_{k,j}(x') = \sum_{k=0}^\infty \lambda_k N(d,k)P_{k,d}(x^\top x'),
 $$
 where the last step uses \eqref{eqn: dot-2}. 
 We complete the proof of the second part by choosing
\begin{align*}
      \sigma(t) = \sum_{k=0}^\infty \sqrt{\lambda_k}N(d,k)P_{k,d}(t).
\end{align*}
\end{proof}
Proposition \ref{integral} implies that: 1) any kernel admits a random feature representation with $\cV = \cX$ and a symmetric feature function $\phi(\cdot,\cdot)$; 2) dot-product kernels further exhibit dot-product structures in their feature functions. Consequently, we can focus on random feature kernels induced by symmetric features without loss of generality.
%This specific form enables us to establish a connection between the $L^\infty$ learning of RKHS and the linear approximation of single neurons (see Section 4 for details).
% We will consider specific examples of dot-product kernels by identifying either $\kappa(\cdot)$ or $\sigma(\cdot)$ in Section 3.1.

\section{Main Results}
\label{sec: main-result}

Let $\cH_k$ be the RKHS associated kernel $k$ on the input domain $\cX$. Suppose that the training data $S_n=\{(x_i,y_i)\}_{i=1}^n$ are generated by $y_i = f(x_i) + \xi_i$, where input data $\{x_i\}_i$ are independently sampled from the input distribution $\rho \in \cP(\cX)$, the target function $f$ lies in the unit ball of the RKHS, i.e., $\|f\|_{\cH_k} \leq 1$, and $\{\xi_i\}_i$ are \iid  sub-gaussian noise. We consider the problem of learning $f$ from the training data $S_n$ under the $L^\infty$ metric.

In this section, we will show that the  error of $L^\infty$ is closely related to the eigenvalue decay of  $k(\cdot,\cdot)$, which can be quantified by 
\begin{equation}
\Lambda_{k,\pi}(n) =  \sqrt{\sum_{i=n+1}^\infty \mu_i^{k,\pi}}. 
\end{equation}
In particular, 
\begin{itemize}
      \item First, for dot product on the unit sphere, the quantity $\Lambda_{k,\tau_{d-1}}(n)$ controls the $L^\infty$ generalization error of the standard KRR estimator.
    \item Second, $\sup_{\pi \in \cP(\cX)} \Lambda_{k,\pi}(n)$ provides a lower bound on the minimax error for learning functions in $\cH_k$ under $L^\infty$ metric.
\end{itemize}

\subsection{Upper Bounds}
\begin{theorem}[Upper bound] \label{upper bound}
Suppose $k:\SS^{d-1}\times \SS^{d-1}\mapsto\RR$ is a dot-product kernel taking the form of \eqref{eqn: dot-product} and the input distribution is $\rho = \tau_{d-1}$. For any decreasing function $L : \NN^+ \to \RR^+$ that satisfies $\Lambda_{k,\tau_{d-1}}(m) \leq L(m)$, let $q(d,L) = \sup_{k \geq 1}\frac{L(k)}{L((d+1)k)}$. Assume that the noise $\xi_i$'s are mean-zero and $\varsigma$-subgaussian. Consider the KRR estimator
\begin{align*}
       \hat{f}_n = \argmin_{\|\hat{f}\|_{\cH_k} \leq 1} \fn\sum_{i=1}^n(\hat{f}(x_i)-y_i)^2.
 \end{align*}
 Then with probability at least $1-\delta$ over the sampling of $\{(x_i,y_i)\}_{i=1}^n$, we have
 \begin{align} \label{eqn: Linfty-learning-upper}
             \|\hat{f}_n - f \|_{\infty} \lesssim    \inf_{m \geq 1} \left[\sqrt{q(d,L)L(m)} + \sqrt{m}(\epsilon(n,\varsigma,\delta) + e(n,\delta))\right],
 \end{align}
 where $\epsilon(n,\varsigma,\delta) = \left(\frac{\varsigma^2\kappa(1)\left(1+\log(1/\delta)\right)}{n}\right)^{1/4},\, e(n,\delta) = \sqrt{\frac{\kappa(1)\left(\log n + \log(1/\delta)\right)}{n}}.$
\end{theorem}
To obtain the tightest bound, one can choose $L(m) = \Lambda_{k,\tau_{d-1}}(m)$. However, since the exact value of $\Lambda_{k,\tau_{d-1}}(m)$ is often unknown, the introduction of $L(m)$ is mainly for the convenience of calculating the constant $q(d,L)$.

\begin{remark}
 The rotational invariance assumption plays a critical role in the analysis presented above, and the result may potentially be extended to settings where the densities of $\pi$ and $\rho$ are strictly positive. In addition, using localization techniques (see, e.g., \cite[Chapter 13]{wainwright2019high}) to address noise-induced errors might yield tighter bounds.
However, our focus here is understanding how the error rate depends on the kernel spectrum and if the rate suffers from the curse of dimensionality,  not obtaining optimal rates.
\end{remark}

Take $L(m)=\Lambda_{k,\tau_{d-1}}(m)=\sum_{j=m+1}^\infty \mu_j$. If $\mu_j \sim j^{-(1+2\beta)}$,\footnote{Here, just as an illustration, we assume $\mu_j \sim j^{-(1+2\beta)}$, ignoring the effect of multiplicity.} then we roughly have that $L(m)\sim m^{-2\beta}$ and $q(d,L)\sim d^{2\beta}$. Plugging them into \eqref{eqn: Linfty-learning-upper} yields
\begin{equation}\label{eqn: 20}
\|\hat{f}_n - f \|_{\infty} \lesssim_{\beta,\varsigma,\delta}\inf_{m\geq 1}\left(d^\beta m^{-\beta}+ m^{1/2}n^{-1/4}\right)\lesssim_{\beta,\varsigma,\delta} d^{\frac{\beta}{2\beta+1}}n^{-\frac{\beta}{2(2\beta+1)}},
\end{equation}
where we hide constants that depend on $\beta,\varsigma$ and $\delta$. 
 {\em It is evident that if $\beta$ does not depend on $d$, then the $L^\infty$ error does not exhibit the curse of dimensionality, i.e., polynomially many samples are sufficient to recovery the target function in $L^\infty$ metric.}

\paragraph*{Examples.}
We now instantiate the  upper bound established above for concrete examples. Specifically, we discuss how the smoothness of $\sigma(\cdot) $ affects the $L^\infty$ learnability.
\begin{proposition}\cite[Proposition 9]{Wu2021ASA}  \label{smooth-example}
Suppose that for any $m \in \NN_+$, $\sigma(\cdot)$ satisfies $$\sup_{t \in [-1,1]} |\sigma^{(m)}(t)| \lesssim \Gamma(m+1).$$
Then, we have $\Lambda_{k,\tau_{d-1}}(m) \lesssim 1/m$.
\end{proposition}
By asserting $L(m) = C/m$ in Proposition \ref{upper bound}, we immediately obtain that the $L^\infty$ error scales as $d^{1/4}n^{-1/8} $, breaking the curse of dimensionality.

The condition in Proposition \ref{smooth-example} could include 
%\begin{itemize}
%\item Dot-product kernel with smooth $\kappa$. For example,
% \begin{itemize}
% \item 
%Gaussian kernel $k(x,x') = \exp\left(\frac{\|x-x'\|^2}{2}\right)$. Here, %$\kappa(t) = \exp(2-2t) $ satisfying 
%$$
%|\kappa^{(k)}(t)| = 2^k \exp(2-2t) \lesssim 2^k \lesssim %\Gamma(k+1),\,\forall t \in [-1,1] .
%$$
%\item Laplace kernel  $k(x,x') = \exp\left(\frac{\|x-y\|}{2}\right)$. Here, $\kappa(t) = \exp(\sqrt{2-2t}) $.
% \item Sigmoid kernel $k(x,x') = \tanh(1+x^\top x')$. Here $\kappa(t) = \tanh(1+t) $. We refer to \cite[section 4.2]{Wu2021ASA} for the verification of $\sup_{t\in[-1,1]}|\kappa^{(k)}(t)| \lesssim \Gamma(k+1)$. 
% \end{itemize}
%\item
random feature kernels associated with smooth activation functions such as sigmoid, softplus, arctan, GELU \cite{hendrycks2016gaussian}, and Swish/SiLU \cite{elfwing2018sigmoid,ramachandran2017searching}. We refer to \cite[section 4.2]{Wu2021ASA} for the detailed verification of why they satisfy the condition in Proposition \ref{smooth-example}.

\subsection{Lower Bounds}
\begin{theorem}[Lower bound] \label{lower bound}
Let $s = \int_\cX k(x,x) \d \rho(x) $ be the trace of the kernel $k$. Suppose that $\xi_i \sim \cN(0,\varsigma)$. For any estimator $\hat{f}_n$ that maps the training data $\{x_i,y_i\}_{i=1}^n$ to a function on $\cX$ have
\begin{align}\label{eqn: low-1}
    \inf_{\hat{f}_n}\sup_{\|f\|_{\cH_k} \leq 1} \EE \|\hat{f}_n - f \|_\infty \gtrsim \min\left(1, \frac{\varsigma}{\sqrt{s}} \right) \sup_{\pi \in \cP(\cX)}\Lambda_{k,\pi}(n),
\end{align}
where the expectation is taken over the random sampling of $\{(x_i, y_i)\} = \{(x_i,f(x_i)+\xi_i)\}_{i=1}^n$.
\end{theorem}
The above theorem shows that up to a multiplicative constant,  the $L^\infty$ minimax error is lower bounded by $\Lambda_{k,\pi}(n)$. If $\mu_j\sim  j^{-1-\beta_d}$, then $\Lambda_{k,\pi}(n)\sim \beta_d^{-1}n^{-\beta_d}$ and the trace  $s\sim \beta_d^{-1}$. In this case,  the lower bound becomes $O(\beta_d^{-1}n^{-\beta_d})$. If $\beta_d=1/\poly(d)$, then the lower bound exhibit a curse of dimensionality. 

\paragraph*{Comparison with \cite{Kuo2008MultivariateLA}.} \cite{Kuo2008MultivariateLA} considers the scenario without noise, i.e., $y_i=f(x_i)$  for $i=1,2,\dots, n$ and the inputs are deterministic. 
The $L^\infty$ lower bound in \cite{Kuo2008MultivariateLA} is given by: for any $x_1,x_2,\dots,x_n\in \cX$, it holds that 
\begin{equation}\label{eqn: low-2}
\inf_{\hf_n} \sup_{\|f\|_{\cH_k}\leq 1}\|\hf_n-f\|\gtrsim \sup_{\pi \in \cP(\cX)}\Lambda_{k,\pi}(n),
\end{equation}
where the infimum is taken over all possible estimators that use only information $\{(x_i, f(x_i))\}_{i=1}^n$. 
Note that \eqref{eqn: low-2} holds for any samples $\{x_i\}_{i=1}^n$. Hence, it implies that the same lower bound holds even if we can adaptively query  input samples. However, in \eqref{eqn: low-2}, the worst-case target functions may depends training samples $\{x_i\}_{i=1}^n$. Thus, for different training samples, the hard target functions can be different. In contrast, Theorem \ref{lower bound} holds for the standard  setup of statistical learning and 
the worst-case target functions in \eqref{eqn: low-1} only depend on the input distribution instead of the specific training samples.

\paragraph*{Examples.}
For dot-product kernels, when $\sigma(\cdot)$ is non-smooth, it is often that $\beta_d=1/\poly(d)$ (see  spectral analyses of dot-product kernels in \cite{bietti2021deep,scetbon2021spectral,Wu2021ASA}). As concrete examples, consider the ReLU$^\alpha$ activation: $\sigma(t)=\max(0,t^\alpha)$, where $\alpha\in \ZZ_{\geq 0}$.  The Heaviside step and ReLU function correspond to $\alpha=0$ and $\alpha=1$, respectively. The case of $\alpha>1$ also has many applications \cite{weinan2018deep,li2019better,so2021searching}. In particular, for $\alpha=0,1$, \cite{cho2009kernel} shows
\[
\kappa(t)= \begin{cases}\frac{1}{2 \pi}(\pi-\arccos (t)) & \text { if } \alpha=0 \\ \frac{1}{2 \pi d}\left((\pi-\arccos (t)) t+\sqrt{1-t^2}\right) & \text { if } \alpha=1\end{cases}.
\]
Specifically, 
\cite[Proposition 5]{Wu2021ASA} proves that there exists a constant $C_{\alpha,d}$ depending on $1/d$ polynomially such that 
\[
\Lambda_{k,\tau_{d-1}}(n)\geq C_{\alpha, d} n^{-\frac{2\alpha+1}{d}}.
\]
Combining with Theorem \ref{lower bound}, it is evident that the $L^\infty$ learning in the corresponding RKHS suffers from the curse of dimensionality. 

We also note that the smoothness of $\kappa$ does not necessarily imply the smoothness of the corresponding activation function $\sigma(\cdot)$. Indeed, it is possible for $\kappa$ to be smooth while still suffering from the curse of dimensionality in $L^\infty$ learning. For example, consider the Gaussian kernel $k(x,x') = \exp\left(\frac{-\|x-x'\|^2}{2} \right)$, where $\kappa(t) = \exp(t-1)$ satisfying $|\kappa^{(m)}(t)|\lesssim 1$ for any $m \in \NN_+$ and $t \in [-1,1]$. However, the following proposition shows that the spectral decay $\Lambda_{k,\tau_{d-1}}$ does not admit a polynomial rate:
\begin{proposition} \label{spectral-gaussian}
    Consider the Gaussian kernel $k(x,x') = \exp(-\frac{\|x-x'\|^2}{h^2})$. For any $h>0$, 
    there dose not exist absolute constant $\alpha \in \RR$ and $\beta > 0$ such that
    \begin{equation}\label{cod_gaussian}
        \Lambda_{k,\tau_{d-1}}(m)\lesssim \frac{d^\alpha}{m^\beta}.
    \end{equation}
\end{proposition}
The proof is deferred to Appendix \ref{appendixD}.
\section{Proof Sketch}
In this section, we present an overview of the proofs of Theorem  \ref{upper bound} and Theorem \ref{lower bound}. At a high level, our proofs consist of three main components:
\begin{itemize}
    \item We introduce a quantity called the $L^\infty$-$L^2$ gap to measure the difference between the $L^\infty$ norm and $L^2$ norm of functions in $\cH_k$. We show that this quantity controls both the upper and lower bound for the error of $L^\infty$ learning.
    \item We bridge the connection between the $L^\infty$-$L^2$ gap and the approximation of parametric feature functions $\{\phi(x,\cdot)\}_{x \in \cX} $. This connection allows us to address the $L^\infty$ learning problem by investigating the corresponding approximation problem.
    \item  Building upon \cite{Wu2021ASA}, we relate the approximation of the function class $\{\phi(x,\cdot)\}_{x \in \cX} $ to the spectral decay of the kernel $k$.
\end{itemize}
Combining the above three steps, we obtain spectral-based upper and lower bound for the $L^\infty$ learning of RKHS.
\paragraph*{Relating the $L^\infty$ learnability to the $L^\infty$-$L^2$ gap.} 
% Firstly, we note that the gap between the $L^\infty$ and $L^2$ norm is critical for the $L^\infty$ learnability. 
We define the following quantity to measure the difference between the $L^\infty$  and  $L^2$ norm for functions in $\cH_k$: for any   $\nu \in \cP(\cX)$ and $\epsilon >0$, let
\begin{align} \label{gap}
 \Delta_{\nu,\epsilon} := \sup_{\|f\|_{\cH_k} \leq 1,\,\|f\|_{\nu} \leq \epsilon}  \|f\|_\infty.    
 \end{align}
The following proposition demonstrates that the quantity defined in \eqref{gap} serves as both an upper and lower bound for the error of $L^\infty$ learning:

\begin{proposition} \label{learning-gap}
% \begin{itemize}
% \item 
~~~(1). Suppose that $k$ is a product kernel with the form
\eqref{eqn: dot-product} on $\SS^{d-1}$ and $\rho = \tau_{d-1}$. Assume the noise $\xi_i$'s are mean-zero and $\varsigma$-subgaussian. Consider the KRR estimator 
\begin{align*}
       \hat{f}_n = \argmin_{\|\hat{f}\|_{\cH_k} \leq 1} \fn\sum_{i=1}^n(\hat{f}(x_i)-y_i)^2.
 \end{align*}
 Then \wp at least $1-\delta$ over the sampling of $\{(x_i,y_i)\}_{i=1}^n$, we have
 \begin{align*} 
             \|\hat{f}_n - f \|_{\infty} \lesssim  \Delta_{\hat{\rho}_n, \epsilon(n,\varsigma,\delta)},
 \end{align*}
 where $\hat{\rho}_n = \frac{1}{n}\sum_{i=1}^n \delta_{x_i}$ is the empirical measure and $\epsilon(n,\varsigma,\delta) = \left(\frac{\varsigma^2\kappa(1)\left(1+\log(1/\delta)\right)}{n}\right)^{1/4}$.

~~~~~(2). Suppose that $\xi_i \sim \cN(0,\varsigma)$.  we have
\begin{align*}
    \inf_{\hat{f}_n}\sup_{\|f\|_{\cH_k} \leq 1} \EE \|\hat{f}_n - f \|_\infty \gtrsim \Delta_{\rho,\varsigma n^{-1/2} }
\end{align*}
where the infimum is taken over all possible estimators and the expectation is taken \wrt the sampling of $\{(x_i, y_i)\}_{i=1}^n$.
% \end{itemize}
\end{proposition}
The proof of Proposition \ref{learning-gap} is deferred to Appendix \ref{appendixA}.
\paragraph*{Relating the $L^\infty$-$L^2$ gap to the approximation of parametric feature functions.}
Given a feature function $\phi: \cX \times \cX \to \RR$, consider a class of parametric  functions $\Phi := \{\phi(x,\cdot): x \in \cX\}$. The closure of the convex, symmetric hull of $\Phi$ is:
$$        \mathcal{G}_\phi =\overline{\left\{\sum_{i=1}^m a_j \phi(x_i,\cdot):\,\sum_{i=1}^m | a_j | \leq 1, x_j \in \cX, m \in \mathbb{N}^{+}\right\}} 
$$
The following proposition shows that the $L^\infty$-$L^2$ gap is closely related to the linear approximability of functions in $\cG_\phi$:

\begin{proposition} \label{dual-lemma}
Recall that by Proposition \ref{integral}, for any  $\pi \in \cP(\cX)$, there exist a $\phi:\cX\times\cX\mapsto\RR$  such that the kernel $k$ could be written in the form  of
$$k(x,x') = \int_\cX \phi(x,v)\phi(x',v) \d \pi(v).$$
For any $\nu \in \cP(\cX)$ and $\epsilon >0$, we have
\begin{align} \label{dual}
     \sup_{\|f\|_{\cH_k} \leq 1,\,\|f\|_{\nu} \leq \epsilon} \|f\|_{\infty} = \sup_{g \in \cG_\phi} \inf_{b \in L^2(\nu)}  \left\|g - \int_{\cX} b(x)\phi(x,\cdot) \d \nu(x) \right\|_\pi + \epsilon \|b\|_{\nu}.
\end{align}
\end{proposition}
The proof is deferred to Appendix \ref{appendixB}. By this proposition, we have the following observations. 
\begin{itemize}
\item Regarding the upper bound, we are interested in the case where $\nu = \hat{\rho}_n = \frac{1}{n}\sum_{i=1}^n\delta_{x_i}$. Then,  the right hand side of \eqref{dual} becomes 
$$   \sup_{g \in \cG_\phi}\inf_{\bm{c} \in \RR^n} \left\|g-\frac{1}{n}\sum_{i=1}^n c_i\phi(x_i,\cdot) \right\|_{\pi} + \frac{\epsilon}{\sqrt{n}}\|\bm{c}\|,   $$
which could be viewed as the error of approximating $\cG_\phi$ with random features. 
\item As for the lower bound, we are interested in the case where $\nu = \rho$. Let $\tilde{k}(v,v') = \int_{\cX}\phi(x,v)\phi(x,v') \d \rho(x)$  and $f_b=\int_{\cX} b(x)\phi(x,\cdot) \d \rho(x)$. By \eqref{eqn: RKHS-rf}, $\|f_b\|_{\cH_{\tk}}\leq \|b\|_\rho$. Then, the right hand side of  \eqref{dual} is equivalent to
\begin{align*}
\sup_{g \in \cG_\phi}\inf_{h \in \cH_{\tilde{k}}} \|g-h\|_{\pi} + \epsilon \|h\|_{\cH_{\tilde{k}}}.
\end{align*}
Hence, the right hand side of \eqref{dual} becomes the error of approximating functions in $\cG_\phi$ with the RKHS $\cH_{\tilde{k}}$.
\end{itemize}

%It will be proved in Appendix that the left hand side of \eqref{dual} governs the $L^\infty$ estimation error when learning functions in $\cH_k$. Note that the right hand side quantifies the error of approximating the convex hull of single neurons $\{\phi(x,v)\}_{x \in \cX}$ with features $\{\sigma(x_i,v)\}_{i=1}^n$. 
\paragraph*{Relating the approximation of parametric feature functions to the spectral decay of kernel.}
The linear approximation of the function class $\cG_\phi$ with optimal features is extensively studied in \cite{Wu2021ASA}, where both lower bounds and upper bounds are established by leveraging spectral decay. Specifically, \cite{Wu2021ASA} focuses on the linear approximation with optimal features, which are the spherical harmonics for dot-product kernels. However, the problem we are concerned (the right hand side of \eqref{dual}) is the linear approximation with random features or RKHS, rather than the optimal features. To fill this gap, we provide the following proposition.

\begin{proposition}\label{approximation-Barron}
% \begin{itemize}[leftmargin=1.5em, itemsep=0pt]
%     \item 
  ~~~(1).  Suppose $k:\SS^{d-1}\times \SS^{d-1}\mapsto\RR$ is a dot-product kernel taking the form $$k(x,x') =\kappa(x^\top x') = \int_{\SS^{d-1}}\sigma(v^\top x)(v^\top x') \d \tau_{d-1}. $$
    For any decreasing function $L : \NN^+ \to \RR^+$ that satisfies $\Lambda_{k,\tau_{d-1}}(m) \leq L(m)$, let $q(d,L) = \sup_{k \geq 1}\frac{L(k)}{L((d+1)k)}$.  For any $\epsilon > 0$, suppose that $x_i \stackrel{iid}{\sim}\tau_{d-1}$, with probability at least $1-\delta$, we have
  \begin{align}\label{RFMapp}
  \sup_{g \in \cG_\phi} \inf_{\bm{c} \in \RR^n} \left\| g- \sum_{i=1}^n c_i\sigma(x_i^\top \cdot) \right\|_{\tau_{d-1}} + \frac{\epsilon}{\sqrt{n}}\|\bm{c}\| \lesssim \inf_{m \geq 1}\left[\sqrt{q(d,L)L(m)} + \sqrt{m}(\epsilon+e(n,\delta)) \right],   
  \end{align}
where $e(n,\delta) = \sqrt{\frac{\kappa(1)\left(\log n + \log(1/\delta)\right)}{n}}$.
% \item 

~~~~~(2). Recall that by Proposition \ref{integral}, for any  $\pi \in \cP(\cX)$, there exist a $\phi:\cX\times\cX\mapsto\RR$  such that the kernel $k$ could be written in the form  of
$$k(x,x') = \int_\cX \phi(x,v)\phi(x',v) \d \pi(v).$$
Let $s = \int_{\cX}k(x,x)\d \rho(x)$. Define $\tilde{k}(v,v') = \int_{\cX}\phi(x,v)\phi(x,v') \d \rho(x)$ and let $\cH_{\tilde{k}}$ be the corresponding RKHS. For any $\epsilon > 0$ and any positive integer $n$, it holds that 
\begin{align}  \label{kernel-app-lower}
  \sup_{g \in \cG_\phi}\inf_{h \in \cH_{\tilde{k}}} \|g-h\|_{\pi} + \epsilon \|h\|_{\cH_{\tilde{k}}} \geq \min\left(1, \epsilon \sqrt{\frac{n}{s}} \right) \Lambda_{k,\pi}(n).
\end{align}
% \end{itemize}
\end{proposition}
The proof is deferred to Appendix \ref{appendixC}.

Finally, the proofs of Theorem \ref{upper bound} and Theorem \ref{lower bound} can be completed by combining Proposition \ref{learning-gap}, Proposition \ref{dual-lemma}, and Proposition \ref{approximation-Barron}. It is worth noting that the choice of $\pi$ in Proposition \ref{dual-lemma} and Proposition \ref{approximation-Barron} is arbitrary. Therefore, for any $\pi \in \cP(\cX)$, we obtain a lower bound, and we can take the supremum over $\pi$ to obtain the final lower bound.

\section{Conclusion}
In conclusion, we present a spectral-based analysis of  the $L^\infty$ learnability of RKHS. The upper bound result demonstrates that standard KRR algorithms can effectively avoid the curse of dimensionality when dealing with kernels with smooth activation functions. This means that the $L^\infty$ learning in RKHS remains feasible with polynomial sample complexity in these cases. On the other hand, the lower bound analysis considers the standard statistical learning setting, where the hardness of functions is dependent on the input distribution rather than the specific training samples. The results reveal that the $L^\infty$ learning in RKHS suffers from the curse of dimensionality when the kernel spectrum decays slowly, specifically as $\lambda_k \sim k^{-1-\beta}$ with $\beta = 1/\text{poly}(d)$. This phenomenon is observed for the dot-product kernel associated with the ReLU activation function. These findings contribute to a better understanding of the capabilities and limitations of kernel-based learning algorithms in safty- and security-critical applications.

%\begin{proof}
%\end{proof}
%Next, we derive the bound in Theorem \ref{lower bound} and Theorem \ref{upper bound} by investigating the right hand side of \eqref{dual}. 

\bibliographystyle{amsalpha}
\bibliography{ref}
\newpage
\appendix
\section{Proof of Proposition \ref{learning-gap}} \label{appendixA}

\paragraph*{Part I: The upper bound.}
We firstly bound the empirical loss $\frac{1}{n} \sum_{i=1}^n(\hat{f}(x_i)-f(x_i))^2$. By the optimality of the estimator $\hat{f}$, we have
\begin{align}\label{Linfty-gap-bound}
         \fn \sum_{i=1}^n(\hat{f}(x_i) - f(x_i) + \xi_i)^2 \leq \fn\sum_{i=1}^n \xi_i^2.
\end{align}
Hence, 
\begin{align*}
      \frac{1}{n}\sum_{i=1}^n (\hat{f}(x_i) - f(x_i))^2 &\leq \frac{2}{n}\sum_{i=1}^n \xi_i(f(x_i) - \hat{f}(x_i)) \\
       & \leq \frac{2}{n}\sup_{\|f\|_{\cH_k} \leq 2 }\sum_{i=1}^n \xi_if(x_i) \\
       & \leq \frac{2}{n}\sup_{\|f\|_{\cH_k} \leq 2 }\langle \xi_ik(x_i,\cdot) , f\rangle_{\cH_k} \qquad (\text{use } f(x)=\langle f, k(x,\cdot)\rangle_{\cH_k})\\
       & \leq \frac{4}{n}\left\|\sum_{i=1}^n \xi_i k(x_i,\cdot) \right\|_{\cH_k} \\
       &  = \frac{4}{n}\sqrt{\sum_{i=1}^n\sum_{j=1}^n \xi_i\xi_j k(x_i,x_j)} \\
       &  \leq \frac{4}{n} \sqrt{z^\top K z},
\end{align*}
where $K \in \RR^{n \times n}$ is the kernel matrix given by $K_{ij} = k(x_i,x_j) $ and $z=(\xi_1,\xi_2,\dots,\xi_n)$. Note that 
\begin{align*}
    \EE[z^\top K z] =\sum_{i=1}^n \EE[\xi_i^2]k(x_i,x_i)\leq n\varsigma^2\kappa(1).
\end{align*}
 Note $|K_{ij}| \leq \kappa(1)$. By Hanson-Wright inequality \cite[Theorem 6.2.1]{vershynin2019high}, \wp at least $1-\delta$ over the noise $\{\xi_i\}_{i=1}^n$, we have
\begin{align*}
     z^\top K z \lesssim  n\varsigma^2\kappa(1)\left( 1 +  \log(1/\delta)\right)
\end{align*}
and thus
\begin{align*}
    \sqrt{\frac{1}{n}\sum_{i=1}^n \left(\hat{f}(x_i) - f(x_i) \right)^2} \lesssim \epsilon(n,\varsigma,\delta):= \left(\frac{\varsigma^2\kappa(1)\left(1+\log(1/\delta)\right)}{n}\right)^{1/4}.
\end{align*}
Then, by the definition of the $L^\infty$-$L^2$ gap, the $L^\infty$ estimation error is bounded by
\begin{align*}
 \| \hat{f} - f\|_\infty \lesssim \sup_{\|f\|_{\cH_k} \leq 1,\,\|f\|_{\hat{\rho}_n} \leq \epsilon(n,\varsigma,\delta) } \|f\|_\infty = \Delta_{\hat{\rho}_n,\epsilon(n,\varsigma,\delta) }.
\end{align*}
We complete the proof.
\qed
\paragraph*{Part II: The lower bound.}
For any function $f \in \cH_k$, let $\cD_{f}^n$ denote the law of $\{(x_i,f(x_i)+\xi_i)\}_{i=1}^n$.
For a given estimator $\hat{f}_n: (\cX \times \RR)^n \to \cH_k$ and a target function $f \in \cH_k$, the performance of $\hat{f}_n$ is measured by 
\begin{align*}
    d(\hat{f}_n, f) := \| \hat{f}_n(\{ (x_i,f(x_i)+ \xi_i)\}_{i=1}^n) - f \|_\infty.  
\end{align*}
We apply the two-point Le Cam's method \cite[Chapter 15.2]{wainwright2019high} to derive the minimax lower bound. For any $f_1,f_2 \in \{f\in \cH_k: \|f\|_{\cH_k} \leq 1\}$, we have
\begin{align*}
    \inf_{\hat{f}_n}\sup_{\|f\|_{\cH_k} \leq 1} \EE [d(\hat{f}_n,f)] & \geq \inf_{\hat{f}_n}\max
    \left\{\EE [d(\hat{f}_n,f_1)],  \EE  [d(\hat{f}_n,f_2)]  \right\} \\
    & \geq \frac{\|f_1-f_2\|_\infty}{2}\Big(1-  \mathrm{TV}(\cD_{f_1}^n,\cD_{f_2}^n) \Big). 
\end{align*}
Applying the Pinsker's inequality~\footnote{For two probability distributions $P,Q$ defined on the same domain, $\|P-Q\|_{\mathrm{TV}}\leq \sqrt{\mathrm{KL}(P||Q)/2}$.} \cite[Lemma 15.2]{wainwright2019high}, we obtain
\begin{align}   \label{1234}
     \inf_{\hat{f}_n}\sup_{\|f\|_{\cH_k} \leq 1} \EE [d(\hat{f}_n,f)]& \geq \frac{\|f_1-f_2\|_\infty}{2}\left(1- \sqrt{\frac{\mathrm{KL}(\cD_{f_1}^n\| \cD_{f_2}^n)}{2}} \right) .
\end{align}
The KL divergence between $\cD_{f_1}^n$ and $\cD_{f_2}^n$ can be computed by
\begin{align}
    \mathrm{KL}(\cD_{f_1}^n \| \cD_{f_2}^n ) = & \EE_{\cD_{f_1^n}}\left[\log \frac{\d \cD_{f_1}^n }{\d \cD_{f_2}^n}\right] \notag  \\
    & = \EE_{\cD_{f_1}^n}\left[\log  \left( \frac{\prod_{i=1}^n \rho(x_i)\exp\left(-\frac{\|\xi_i\|^2}{2\varsigma^2} \right) }{\prod_{i=1}^n \rho(x_i)\exp\left(-\frac{\|f_2(x_i) +\xi_i - f_1(x_i)\|^2}{2\varsigma^2} \right)}\right) \right] \notag \\
    & = \EE_{x_i\sim \rho,\,\xi_i\sim \cN(0,\varsigma^2I_d)} \left[\sum_{i=1}^n \left(\frac{\|f_2(x_i)-f_1(x_i)+\xi_i\|^2 - \|\xi_i\|^2}{2\varsigma^2} \right)\right] \notag \\
    & =\frac{n}{2\varsigma^2} \|f_2 - f_1\|_\rho^2. \label{2345}
\end{align}
Combining \eqref{1234} and \eqref{2345}, we arrive at
\begin{align} \label{3456}
  \inf_{\hat{f}_n} \sup_{\|f\|_{\cH_k} \leq 1} \EE[d(\hat{f}_n,f)] \geq \frac{\|f_1-f_2\|_\cM}{2}\left(1 - \frac{\sqrt{n}\|f_2-f_1\|_\rho}{2\varsigma} \right)
\end{align}
There must exist a function $\tf\in \cH_k,\,\|\tf\|_{\cH_k} \leq 1$ such that 
\begin{align*}
         \|\tf \|_{\rho} \leq \frac{\varsigma}{\sqrt{n}},          \quad    \|\tf\|_{\cM} \geq \frac{2}{3}\Delta_{\rho,\varsigma n^{-1/2}}.
\end{align*}
Substituting $f_1 = 0, f_2 = \tf$ into \eqref{3456} yields
\begin{align*}
     \inf_{\hat{f}_n}\sup_{\|f\|_{\cH_k} \leq 1} \EE [d(\hat{f}_n,f)] \geq \frac{1}{6} \Delta_{\rho,\varsigma n^{-1/2}}.
\end{align*}
We complete the proof.
\qed
\section{Proof of Proposition \ref{dual-lemma}} \label{appendixB}
We firstly need the following lemma.
\begin{lemma} \label{dual-lemma-lemma}
Let $\phi:\cX\times\cX\mapsto\RR$ and $\nu,\pi \in \cP(\cX)$ and $\epsilon > 0$. Define
\begin{align*}
    A = \left\{a \in L^2(\pi): \|a\|_{\pi} \leq 1,\, \left\|\int_\cX a(v)\phi(\cdot,v) \d \pi(v) \right\|_{\nu}  \leq \epsilon \right\}.  
\end{align*}
For any function $g \in L^{2}(\pi)$, we have
    \begin{align*}
    \sup_{a \in A} \int_\cX a(v)g(v) \d \pi(v) = \inf_{h \in L^{2}(\nu)} \left(\left\|g - \int_\cX h(x) \phi(x,\cdot) \d \nu(x) \right\|_{\pi} + \epsilon \|h\|_{\nu}\right).
\end{align*}
\end{lemma}
\begin{proof}

%, and the only difference is replace the
% The proof is analogous; 
\underline{\bf Step I:}~~
 On the one hand,  we have  for any $h \in L^{2}(\nu)$ that
    \begin{align*}
      &\left\|g - \int_\cX h(x) \phi(x,\cdot) \d \nu(x) \right\|_{\pi} + \epsilon \|h\|_{\nu}  \\ 
      &= \sup_{\|a\|_{\pi} \leq 1}\left\langle a, g - \int_\cX h(x) \phi(x,\cdot) \d \nu(x)\right\rangle_\pi + \epsilon  \|h\|_{\nu}\\
      & =  \sup_{\|a\|_{\pi} \leq 1}\left[ \int_\cX a(v)g(v) \pi(v) - \int_\cX h(x) \left( \int_\cX a(v)\phi(x,v) \d \pi(v)  \right)\d \nu(x) \right] + \epsilon \|h\|_{r',\nu} \\
      & \geq \sup_{a\in A}\left[ \int_\cX a(v)g(v) \pi(v) - 
            \|h\|_{\nu}\left\|\int_\cX a(v)\phi(\cdot,v) \d \pi(v)\right\|_{\nu}\right] + \epsilon \|h\|_{\nu}\\ 
      &\geq \sup_{a\in A}\int_\cX a(v)g(v) \pi(v),
    \end{align*}
    where the last step uses the property that $\left\|\int_\cX a(v)\phi(\cdot,v) \d \pi(v)\right\|_{\nu}\leq \epsilon$ for $a\in A$.
    Hence,
        \begin{align}\label{eqn: 008}
    \sup_{a \in A} \int_\cX a(v)g(v) \d \pi(v) \leq \inf_{h \in L^{2}(\nu)}\left( \left\|g - \int_\cX h(x) \phi(x,\cdot) \d \nu(x) \right\|_{\pi} + \epsilon \|h\|_{\nu}\right).
\end{align}

\underline{\bf Step II:}~~ On the other hand, consider the functional $\cT: L^{2}(\pi)\mapsto\RR$ given by 
\begin{align*}
    \cT(l) := \inf_{h \in L^{2}(\nu)} \left(\left\|l - \int_\cX h(x) \phi(x,\cdot) \d \nu(x) \right\|_{\pi} + \epsilon \|h\|_{\nu}\right).
\end{align*}
It is not hard to verify that 
\begin{itemize}
\item $\cT$ is sublinear on $L^{2}(\pi)$, i.e., $\cT(l_1+l_2)\leq \cT(l_1)+\cT(l_2)$,
\item $\cT(\lambda l)=\lambda \cT(l)$ for any $l\in L^{2}(\pi)$.
\end{itemize}
For a given $g\in  L^2(\pi)$, let $G=\mathrm{span}\{g\}$. By the Hahn-Banach Theorem \cite[Section 5]{folland1999real}, there exists a linear functional $\tilde{\cT}$ on $L^{2}(\pi)$ such that 
\begin{itemize}
\item $\tilde{\cT}(g) = {\cT}(g)$, i.e., 
\begin{align*}
   \tilde{\cT}(g) =  \inf_{h \in L^{2}(\nu)} \left\|g - \int_\cX h(x) \phi(x,\cdot) \d \nu(x) \right\|_{\pi} + \epsilon \|h\|_{\nu};
\end{align*}
\item $\tilde{\cT}(l) \leq {\cT}(l)$ for any $l \in L^{2}(\pi)$, i.e.,
\begin{align} \label{6666}
      \tilde{\cT}(l) \leq \inf_{h \in L^{2}(\nu)} \left\|l - \int_\cX h(x) \phi(x,\cdot) \d \nu(x) \right\|_{\pi} + \epsilon \|h\|_{\nu},\quad \text{for any}\, l \in L^{2}(\pi).
\end{align}
\end{itemize}

Taking $h=0$ in \eqref{6666} gives 
\[
\tilde{\cT}(l) \leq \|l\|_{\pi} \quad \text{for any}\, l \in L^{2}(\pi).
\]
Hence, $\|\tilde{\cT}\|_{\mathrm{op}}\leq 1$. By Riesz representation theorem, there exists $a_g \in L^2(\pi)$ with $\|a\|_{\pi} \leq 1$ such that 
\begin{align}\label{eqn: 009}
    \tilde{\cT}(l) = \int_\cX a_g(v)l(v) \d \pi(v),\quad \text{for any}\, l \in L^{2}(\pi),
\end{align}
where $a_g(\cdot)$ depends on $g$ as $\tilde{\cT}$ depends on $g$. 

Noticing that for any $ h\in L^{2}(\nu)$,
\begin{align*}
  \int_\cX h(x)\left(\int_\cX a_g(v)\phi(x,v)\d \pi(v)\right)\d \nu(x)  =  \tilde{\cT}\left(\int_\cX h(x)\phi(x,\cdot) \d \nu(x) \right) \leq \epsilon \|h\|_{\nu},
\end{align*}
we can conclude that $a_g(\cdot)$ satisfies 
$
\left\|\int_\cX a_g(v)\phi(\cdot,v)\d \pi(v)\right\|_{\nu}\leq \epsilon,
$
implying 
\[
a_g\in A.
\]
Thus, 
    \begin{align}\label{eqn: 007}
\notag    \sup_{a \in A} \int_\cX a(v)g(v) \d \pi(v)& \geq  \int_\cX a_g(v)g(v) \d \pi(v)=\tilde{\cT}(g)=\cT(g)\\ 
    &=\inf_{h \in L^2(\nu)} \left(\left\|g - \int_\cX h(x) \phi(x,\cdot) \d \nu(x) \right\|_{\pi} + \epsilon \|h\|_{\nu}\right)
    \end{align}
By combining \eqref{eqn: 008} and \eqref{eqn: 007}, we complete the proof.
\end{proof}

\underline{\bf Proof of Proposition \ref{dual-lemma}.}~~
We note that the function class $\cG_\phi$ could be written as
$$\cG_\phi = \left\{\int_{\cX}a(x)\phi(x,\cdot)\d \mu(x):\,\mu \in \cM(\cX),\,\|\mu\|_{\mathrm{TV}} \leq 1 \right\}  $$
and the RKHS ball $\cH_k^1: = \{f \in \cH_k: \|f\|_{\cH_k} \leq 1\}$ could be written as
$$ \cH_k^1 = \left\{\int_\cX a(v)\phi(\cdot,v) \d \pi(v): \|a\|_\pi \leq 1 \right\}.     $$
Using the duality form $\|f\|_{\infty} = \sup\limits_{\mu \in \cM(\cX),\,\|\mu\|_{\mathrm{TV}}} \int_\cX f(x) \d \mu(x)$ and combining Lemma \ref{dual-lemma-lemma}, we have
\begin{align*}
    \sup_{\|f\|_{\cH_k} \leq 1,\|f\|_{\nu} \leq \epsilon} \|f\|_{\infty} &= \sup_{\|f\|_{\cH_k} \leq 1,\|f\|_{\nu} \leq \epsilon} \sup_{\mu \in \cM(\cX),\,\|\mu\|_{\mathrm{TV}} \leq 1} \int_\cX f(x)\d \mu(x) \\
   & =  \sup_{\mu \in \cM(\cX),\,\|\mu\|_{\mathrm{TV}} \leq 1} \sup_{a \in A} \int_\cX \left(\int_\cX a(v)\phi(x,v) \d \pi(v) \right) \d \mu(x)  \\    
    & =\sup_{\mu \in \cM(\cX),\,\|\mu\|_{\mathrm{TV}} \leq 1} \sup_{a \in A} \int_\cX a(v) \left(\int_\cX \phi(x,v) \d \mu(x) \right)\d \pi(v)  \\
    & \stackrel{(i)}{=} \sup_{\mu \in \cM(\cX),\,\|\mu\|_{\mathrm{TV}} \leq 1} \inf_{h \in L^{2}(\nu)} \left\|\int_\cX \phi(x,\cdot)\d \mu(x) - \int_\cX h(x)\phi(x,\cdot)\d \nu(x)\right\|_{\pi} + \epsilon \|h\|_{\nu} \\
    & = \sup_{g \in \cG_\phi}\inf_{h \in L^{2}(\nu)} \left\|g - \int_\cX h(x)\phi(x,\cdot)\d \nu(x)\right\|_{\pi} + \epsilon \|h\|_{\nu}.
\end{align*}  
where (i) comes from Lemma \ref{dual-lemma-lemma}. We complete the proof.
\qed
\qed
\section{Proof of Proposition \ref{approximation-Barron}} \label{appendixC}
\subsection{The Upper Bound}
Our proof needs the following two lemmas related to the linear approximation of parametric function and the random feature approximation of RKHS. 
\begin{lemma} \label{app-kernel}
 For any $x \in \SS^{d-1}$, there exists a function $h \in \cH_k $ such that $\|h\|_{\cH_k} \leq \sqrt{m}$ and 
\begin{align} \label{eq-app-kernel}
   \|\sigma_x -  h \|_{\tau_{d-1}}  \lesssim \sqrt{q(d,L)L(m)}, 
\end{align}
where $\sigma_x :\SS^{d-1} \to \RR$ is the single neuron $v \mapsto \sigma(v^\top x)$.
\end{lemma}
\begin{proof}
Let $m_k = \sum_{i=0}^k N(d,i)$. Assume that $m \in [m_{k},m_{k+1}-1]$, we choose $h$ as the optimal approximation of $\sigma(x^\top \cdot)$ in the span of $\{Y_{i,j}\}_{1\leq i \leq k,\, 1\leq j \leq N(d,i) }$, that is, $\bar{h} = \sum_{i=1}^k \sum_{j=1}^{N(d,i)} c_{i,j}Y_{i,j}$ with 
\begin{align*}
\{c_{i,j}\} =    \argmin_{\{\alpha_{i,j}\}_{1\leq i \leq k,\,1\leq j \leq N(d,i)}} \left\|\sigma_x - \sum_{i=1}^k \sum_{j=1}^{N(d,i)}\alpha_{i,j}Y_{i,j} \right\|_{\tau_{d-1}}^2 
\end{align*}
Now we verify that $\bar{h}(\cdot)$ satisfies the conditions in Lemma \ref{app-kernel}. Recall that the spherical harmonics is related to the Legendre polynomials: 
\begin{align*}
    \sum_{j=1}^{N(d,k)}Y_{k,j}(v)Y_{k,j}(v')  = N(d,k)P_k(v^\top v'),\quad \forall v,v' \in \SS^{d-1}.
\end{align*}
Since $\left\{Y_{i, j}\right\}_{0 \leq i \leq k, 1 \leq j \leq N(d, i)}$ is orthonormal in $L^2\left(\tau_{d-1}\right)$, we have
$$
\inf _{\left\{c_{i, j}\right\}_{0 \leq i \leq k, 1 \leq j \leq N(d, i)}}\left\|\sigma_x-\sum_{i=0}^k \sum_{j=1}^{N(d, i)} c_{i, j} Y_{i, j}\right\|_{\tau_{d-1}}^2=\left\|\sigma_x\right\|_{\tau_{d-1}}^2-\sum_{i=0}^k \sum_{j=1}^{N(d, i)}\left\langle Y_{i, j}, \sigma_x\right\rangle_{\tau_{d-1}}^2
$$
Hence,
$$
\begin{aligned}
&\left\|\sigma_x\right\|_{\tau_{d-1}}^2-\sum_{i=0}^k \sum_{j=1}^{N(d, i)}\left\langle Y_{i, j}, \sigma_x\right\rangle_{\tau_{d-1}}^2\\
&=\int_{\mathbb{S}^{d-1}}|\sigma(x^\top v)|^2 \mathrm{~d} \tau_{d-1}(v)-\sum_{i=0}^k \int_{\mathbb{S}^{d-1}} \int_{\mathbb{S}^{d-1}} \sigma(x^\top v) \sigma(x^\top v^{\prime}) \sum_{j=1}^{N(d, i)} Y_{i, j}(v) Y_{i, j}(v^{\prime}) \mathrm{d} \tau_{d-1}(v) \mathrm{d} \tau_{d-1}(v^{\prime})\\
&=\int_{\mathbb{S}^{d-1}}|\sigma(x^\top v)|^2 \mathrm{~d} \tau_{d-1}(v)-\sum_{i=0}^k N(d, i) \int_{\mathbb{S}^{d-1}} \int_{\mathbb{S}^{d-1}} \sigma(x^\top v) \sigma(x^\top v^{\prime}) P_i(v^\top v^{\prime}) \mathrm{d} \tau_{d-1}(v) \mathrm{d} \tau_{d-1}\left(v^{\prime}\right)
\end{aligned}
$$
By rotation invariance, we have
\begin{align}
   & \int_{\SS^{d-1}} |\sigma(x^\top v) |^2\d \tau_{d-1}(x) =  \int_{\SS^{d-1}}\int_{\SS^{d-1}} |\sigma(x^\top v) |^2\d \tau_{d-1}(x) \d \tau_{d-1}(v)  \notag \\ &= \int_{\SS^{d-1}}\kappa(v^\top v) \d \tau_{d-1}(v) = \sum_{i=0}^\infty N(d,i)\lambda_i, \label{invariance-app-1}
\end{align}
and 
\begin{align}
  &  \sum_{i=0}^k N(d, i) \int_{\mathbb{S}^{d-1}} \int_{\mathbb{S}^{d-1}} \sigma(x^\top v) \sigma(x^\top v^{\prime}) P_i(v^\top v^{\prime}) \mathrm{d} \tau_{d-1}(v) \mathrm{d} \tau_{d-1}\left(v^{\prime}\right) \notag \\
  & = \sum_{i=0}^k N(d, i) \int_{\mathbb{S}^{d-1}} \int_{\mathbb{S}^{d-1}} \int_{\mathbb{S}^{d-1}} \sigma(x^\top v) \sigma(x^\top v^{\prime})\mathrm{d} \tau_{d-1}(x)  P_i(v^\top v^{\prime}) \mathrm{d} \tau_{d-1}(v) \mathrm{d} \tau_{d-1}(v^{\prime})   \notag \\
  & = \sum_{i=0}^k N(d, i)  \int_{\mathbb{S}^{d-1}} \int_{\mathbb{S}^{d-1}} \kappa(v^\top v')P_i(v^\top v')  \mathrm{d} \tau_{d-1}(v) \mathrm{d} \tau_{d-1}(v^{\prime}) \notag  \\
  & = \sum_{i=0}^k\sum_{j=1}^{N(d,i)} \int_{\SS^{d-1}}  \int_{\SS^{d-1}} \kappa(v^\top v')  Y_{i,j}(v') \d \tau_{d-1}(v')                       \notag        \\
  & = \sum_{i=0}^k N(d,i)\lambda_i, \label{invariance-app-2}
\end{align}
where the last equation is because $\{Y_{i,j}\}_{1\leq i,j\leq n}$ are the eigenfunctions of the kernel operator: 
$$\int_{\SS^{d-1}}\kappa(v^\top v')Y_{i,j}(v') \d \tau_{d-1}(v') = \lambda_iY_{i,j}(v). $$
Combining \eqref{invariance-app-1} and \eqref{invariance-app-2}, we arrive at
\begin{align*}
              \|\sigma_x - h \|_{\tau_{d-1}}^2 = \sum_{i=k+1}^\infty N(d,i)\lambda_i \leq L(m_k).
\end{align*}
Note that $\frac{m_{k+1}}{m_k} \leq d+1$(see \cite[Proposition 3]{Wu2021ASA} for details), we have
\begin{align*} 
    L(m_k) \leq \frac{L(m_{k+1})}{L(m_k)}L(m)  \leq q(d,L)L(m).
\end{align*}
Thus we complete the proof of \eqref{eq-app-kernel}. To bound the RKHS norm of $\bar{h}$, we note that the optimal approximation is given by
\begin{align*}
    c_{i,j} = \langle \sigma_x, Y_{i,j} \rangle_{\tau_{d-1}}.
\end{align*}
Similar to the proof of \eqref{eq-app-kernel}, we have
\begin{align*}
 &   \left\|\sum_{i=1}^k \sum_{j=1}^{N(d,i)}c_{i,j}Y_{i,j} \right\|_{\cH_k}^2 =  \sum_{i=0}^k \frac{1}{\lambda_i}\sum_{j=1}^{N(d,i)}\langle Y_{i,j}, \sigma_x \rangle_{\tau_{d-1}}^2 \\
    & = \sum_{i=0}^k \frac{1}{\lambda_i}N(d, i) \int_{\mathbb{S}^{d-1}} \int_{\mathbb{S}^{d-1}} \kappa(x^\top x') P_i\left(x^{\mathrm{T}} x^{\prime}\right) \mathrm{d} \tau_{d-1}(x) \mathrm{d} \tau_{d-1}\left(x^{\prime}\right)\\
    & = \sum_{i=1}^k N(d,i) = m_k \leq m.
\end{align*}
We complete the proof.
\end{proof}
\begin{lemma} \label{RFM-app-kernel}
 Suppose that $x_i \stackrel{iid}{\sim}\tau_{d-1}$, with probability at least $1-\delta$, we have    
 $$
 \sup_{\|h\|_{\cH_k} \leq 1} \inf_{c_1,\cdots,c_n} \left\|h-\frac{1}{n}\sum_{i=1}^n c_i\sigma_{x_i} \right\|_{\tau_{d-1}}   +\frac{\epsilon}{\sqrt{n}}\|\bm{c}\|\leq \epsilon + e(n,\delta),    $$
where $e(n,\delta) = \sqrt{\frac{\kappa(1)\left(\log n + \log(1/\delta)\right)}{n}}$.
\end{lemma}
\begin{proof}
We directly apply \cite[Proposition 1]{bach2017equivalence}. With probability at least $1-\delta$ over samples $\{x_i\}_{i=1}^n$, we have
$$
\sup _{\|h\|_{\cH_k} \leq 1} \inf_{ \|\bm{c}\|^2 \leq 4n} \left\|h -\frac{1}{n} \sum_{i=1}^n c_i \sigma_{x_i}  \right\|_{\tau_{d-1}} \leq 4 t
$$
where $t$ satisfies that
\begin{align}\label{12}
5 d(t) \log \left(\frac{16 d(t)}{\delta}\right)\leq n
\end{align}
and $d(t)=\sup _{x \in \SS^{d-1}}\left\langle \sigma_x,(\Sigma+t I)^{-1} \sigma_x \right\rangle_{\tau_{d-1}}$, in which $\Sigma$ is a self-adjoint, positive semi-definite operator on $L^2(\tau_{d-1})$ (see the detailed definition in \cite[Section 2.1]{bach2017equivalence}). 
 Notice that 
\begin{align}\label{13}
\begin{aligned}
d(t) \leq & t^{-1} \sup _{x \in \SS^{d-1}}\langle \sigma_x, \sigma_x \rangle_{\tau_{d-1}}\leq  t^{-1} \kappa(1).
\end{aligned}
\end{align}
From \eqref{12},\eqref{13}, we have
$$
\begin{aligned}
& \sup _{\|h\|_{\cH_k} \leq 1}\inf_{\|\bm{c}\|^2 \leq 4n} \left\|h-\frac{1}{n}\sum_{i=1}^n c_i\sigma_{x_i} \right\|_{\tau_{d-1}}^2 \lesssim \frac{\kappa(1)}{n}\left[1+\log \left(\frac{n}{\delta}\right)\right].
\end{aligned}
$$
The above implies 
\begin{align*} 
 \sup_{\|h\|_{\cH_k} \leq 1} \inf_{c_1,\cdots,c_n} \left\|h - \sum_{i=1}^n c_i\sigma_{x_i} \right\|_{\tau_{d-1}} + \frac{\epsilon}{\sqrt{n}}\|\bm{c}\| \lesssim \left(\epsilon + \sqrt{\frac{\kappa(1)\left(\log n + \log(1/\delta)\right)}{n}} \right).
\end{align*}
We complete the proof.
\end{proof}

\paragraph*{Proof of the first part(upper bound) of Proposition \ref{approximation-Barron}.} 
By triangle inequality, for any $m \in \NN_+ $ we could decompose the left hand side of \eqref{RFMapp} as
\begin{align}
&  \sup_{g \in \cG_\phi } \inf_{\bm{c} \in \RR^n}  \left\|g - \sum_{i=1}^n c_i\sigma_{x_i} \right\|_{\tau_{d-1}} + \frac{\epsilon}{\sqrt{n}}\|\bm{c}\|  \notag \\
& \leq \sup_{g \in \cG_\phi} \inf_{\|h\|_{\cH_k} \leq \sqrt{m}}  \left\|g - h \right\|_{\tau_{d-1}}     \label{decompose1}      \\
& + \sup_{\|h\|_{\cH_k} \leq \sqrt{m}}\inf_{c_1,\cdots,c_n} \left\|h-\frac{1}{n}\sum_{i=1}^n c_i\sigma_{x_i} \right\|_{\tau_{d-1}}   +\frac{\epsilon}{\sqrt{n}}\|\bm{c}\|          \label{decompose2}
\end{align}
Firstly we consider the term  \eqref{decompose1}. By Lemma \ref{app-kernel}, for any $x \in \SS^{d-1}$, there exists a function $h_x \in \cH_k$ such that $\|h_x\|_{\cH_k} \leq \sqrt{m}$ and $\|\sigma_x - h_x\|_{\tau_{d-1}} \leq \sqrt{q(d,L)L(m)} $. For any $g \in \cG_\phi$, there exists a signal measure $\mu \in \cM(\cX),\,\|\mu\|_{\mathrm{TV}} \leq 1$ such that $ g= \int_\cX \sigma_x \d \mu(x)$. Let $h = \int_\cX h_x \d \mu(x) $, by Jensen's inequality we know that $h$ satisfies
$\|h\|_{\cH_k} \leq \int_{\SS^{d-1}} \|h_x\|_{\cH_k} \d \mu(x) \leq \sqrt{m}$ and
 \begin{align} \label{bound1}
          \left\|\int_{\SS^{d-1}}\sigma_x \d \mu(x)  - h \right\|_{\tau_{d-1}}^2 \leq \int_{\SS^{d-1}} \|\sigma_x - h_x \|_{\tau_{d-1}}^2 \d \mu(x) \leq q(d,L)L(m).
 \end{align} 
Thus \eqref{decompose1} is bounded by $\sqrt{q(d,L)L(m)} $. Next, we apply Lemma \ref{RFM-app-kernel} to bound the second term \eqref{decompose2}
\begin{align} \label{bound2}
 \sup_{\|h\|_{\cH_k} \leq \sqrt{m}} \inf_{c_1,\cdots,c_n} \left\|h - \frac{1}{n}\sum_{i=1}^n c_i\sigma_{x_i} \right\|_{\tau_{d-1}} + \frac{\epsilon}{\sqrt{n}}\|\bm{c}\| \lesssim \sqrt{m}\left(\epsilon + \sqrt{\frac{\kappa(1)\left(\log n + \log(1/\delta)\right)}{n}} \right).
\end{align}
Combining \eqref{bound1} and \eqref{bound2} and using the arbitrariness of $m$, we complete the proof.
\qed
\subsection{The Lower Bound}
Our proof needs the following two lemmas related to the linear approximation of RKHS and parametric feature functions
\label{sec: proof-rkhs-upper-bound}
%Let $k$ be a kernel on $\cX$. For any distribution $\pi \in \cP(\cX)$, Recall that Proposition \ref{integral} allows us to represent $k$ with the form:
%\begin{align}\label{integral-representation}
 %    k(x,x') = \int_{\cX} \phi(x,v)\phi(x',v) \d \pi(v),
%\end{align}
%where $\phi: \cX \times \cX \to \RR$ is some symmetric feature function.

\begin{lemma} \label{RFM-app-neuron}\cite[Proposition 1]{Wu2021ASA}
For any basis function $\varphi_1,\cdots,\varphi_n \in L^2(\pi)$, we have
\begin{align*}
\sup_{g \in \cG_\phi}\inf_{c_1,\cdots,c_n} \left\|g - \sum_{j=1}^n c_j \varphi_j \right\|_\pi  \geq \Lambda_{k,\pi}(n).
\end{align*}
\end{lemma}
\begin{lemma}  \label{appRKHS}
For a kernel $\tilde{k}$ on $\cX$, let $\cH_{\tilde{k}}$ be the corresponding RKHS, denote $\tilde{s} = \int_{\cX}  \tilde{k}(x,x) \d \pi(x)$ the trace of $\tilde{k}$ with respect to $\pi$. 
Recall that $\{e_j^{\tilde{k},\pi} \}$ be the eigenfunctions of the kernel $\tilde{k}$ with respect to $\pi$. We have
\begin{align*}
 \sup_{\|h\|_{\cH_{\tilde{k}}}\leq 1}\inf_{c_1,\cdots,c_n} \left\| h -\sum_{j=1}^n c_j e_j^{\tilde{k},\pi} \right\|_\pi  \leq \sqrt{\frac{\tilde{s}}{n}}
\end{align*}
\end{lemma}
\begin{proof}
   Since $\{e_j^{\tilde{k},\pi}\}$ are orthonormal in $L^2(\pi)$, for any $h \in \cH_{\tilde{k}}$ with $\|h\|_{\cH_{\tilde{k}}} \leq 1$, the optimal approximation is given by $c_j = \langle h,e_j^{\tilde{k},\pi} \rangle_\pi$. Thus
   \begin{align}\label{4567}
       \inf_{c_1,\cdots,c_n}\left\|h - \sum_{j=1}^n c_je_j^{\tilde{k},\pi}\right\|_\pi^2 & = \left\|h - \sum_{j=1}^n \langle h,e_j^{\tilde{k},\pi}  \rangle_\pi  e_j^{\tilde{k},\pi} \right\|_\pi^2 \notag  \\
       &  = \sum_{j=n+1}^\infty \langle h, e_j^{\tilde{k},\pi}\rangle_\pi^2. 
   \end{align}
   By the definition of RKHS norm, we have 
   \begin{equation}\label{eqn: 015}
 \|h\|_{\cH_{\tilde{k}}}^2=\sum_{j=1}^\infty \frac{1}{\mu_j^{\tilde{k},\pi}} \langle h,e_j^{\tilde{k},\pi} \rangle_\pi^2  \leq 1.
   \end{equation}
   The right hand side of \eqref{4567} is bounded by
  \begin{align*}
\sum_{j=n+1}^{\infty}\left\langle h, e_j^{\tilde{k},\pi}\right\rangle_\pi^2 \leq \mu_n^{\tilde{k},\pi} \sum_{j=n+1}^{\infty} \frac{1}{\mu_j^{\tilde{k},\pi}}\left\langle h, e_j^{\tilde{k},\pi}\right\rangle_\rho^2 \leq \mu_n^{\tilde{k},\pi} \sum_{j=1}^{\infty} \frac{1}{\mu_j^{\tilde{k},\pi}}\left\langle h, e_j^{\tilde{k},\pi}\right\rangle_\pi^2 \leq  \mu_n^{\tilde{k},\pi},
  \end{align*}
  where the last step is due to \eqref{eqn: 015}. Thus, we complete the proof by noticing that $\mu_n^{\tilde{k},\pi} \leq \frac{\tilde{s}}{n} $, which is due to 
  \[
  \ts =\sum_{j=1}^\infty \mu_j^{\tilde{k},\pi}\geq \sum_{j=1}^n \mu_j^{\tilde{k},\pi}\geq n \mu_n^{\tilde{k},\pi}.
  \]
\end{proof}

\paragraph*{Proof of the second part(lower bound) of Proposition \ref{approximation-Barron}.} 
 By the triangle inequality, we decompose \eqref{kernel-app-lower} as
    \begin{align}
&    \sup_{g \in \cG_\phi}\inf_{h \in \cH_{\tilde{k}}} \left\|g- h \right\|_\pi + \epsilon \|h\|_{\cH_{\tilde{k}}}  \notag \\
& \geq \sup_{g \in \cG_\phi}\inf_{\bm{c} \in \RR^n} \left\|g - \sum_{j=1}^n c_j e_j^{\tilde{k},\pi} \right\|_\pi  \label{decompose3} \\ 
& 
-\left( \sup_{h \in \cH_{\tilde{k}}}\inf_{\bm{c} \in \RR^n} \left\| h-\sum_{j=1}^n c_j e_j^{\tilde{k},\pi} \right\|_\pi     - \epsilon \|h\|_{\cH_{\tilde{k}}} \right)    \label{decompose4}
    \end{align}
Applying Lemma \ref{RFM-app-neuron}, the first term \eqref{decompose3} can be bounded by
\begin{align*}
     \sup_{h \in \cG_\phi}\inf_{\bm{c} \in \RR^n} \left\|h - \sum_{j=1}^n c_j e_j^{\tilde{k},\pi} \right\|_\pi \geq \Lambda_{{k},\pi}(n).
\end{align*}
To tackle the second term \eqref{decompose4}, we use Lemma \ref{appRKHS} to obtain
\begin{align*}
     \sup_{h \in \cH_{\tilde{k}}}\inf_{c_1,\cdots,c_n} \left\| h -\sum_{j=1}^n c_j e_j^{\tilde{k},\pi} \right\|_\pi     - \epsilon \|h\|_{\cH_{\tilde{k}}}  \leq 0,\,\forall \epsilon \geq \sqrt{\frac{\tilde{s}}{n}},
\end{align*}
where $\tilde{s} = \int_{\cX}\tilde{k}(v,v)\d \pi(v) $ be the trace of $\tilde{k}$. Combining the bound for \eqref{decompose3} and \eqref{decompose4}, we obtain
$$
  \sup_{g \in \cG_\phi}\inf_{h \in \cH_{\tilde{k}}} \left\|g- h \right\|_\pi + \epsilon \|h\|_{\cH_{\tilde{k}}} \geq \Lambda_{k,\pi}(n),\quad \forall \epsilon \geq \sqrt{\frac{\tilde{s}}{n}}.
$$
Therefore, for any $\epsilon > 0$, 
\begin{align*}
 \sup_{g \in \cG_\phi}\inf_{h \in \cH_{\tilde{k}}} \left\|g- h \right\|_\pi + \epsilon \|h\|_{\cH_{\tilde{k}}}  & \geq \min\left( 1, \epsilon \sqrt{\frac{n}{\tilde{s}}}\right)\left( \sup_{g \in \cG_\phi}\inf_{h \in \cH_{\tilde{k}}} \left\|g- h \right\|_\pi + \sqrt{\frac{\tilde{s}}{n}} \|h\|_{\cH_{\tilde{k}}}  \right) \\
 & \geq \min\left( 1, \epsilon \sqrt{\frac{n}{\tilde{s}}}\right)\Lambda_{k,\pi}(n)
\end{align*}
Finally, we complete the proof by noticing that 
\begin{align*}
    s =\int_\cX k(x,x) \d \rho(x) =\int_{\cX} \int_{\cX}\phi(x,v)\phi(x,v) \d \pi(v) \d \rho(x) = \tilde{s}.
\end{align*}
\qed
\section{Proof of Proposition \ref{spectral-gaussian}} \label{appendixD}
We first choose a fixed $j \in \mathbb{N}^+$ such that $\beta j > \alpha -1$. By Theorem 2 in \cite{minh2006mercer} and Stirling's formula, we know that
\begin{equation*}
    \lambda_j \sim (\frac{2e}{h^2})^j \frac{d^{d/2+1/2}}{(2j+d-2)^{j+d/2-1/2}} \sim d^{1-j}.
\end{equation*}
Then, again using Stirling's formula, we have
\begin{align*}
    N(d,j)\lambda_j &= \frac{2j+d-2}{j}\binom{j+d-3}{d-2}\lambda_j\\
    &\sim \frac{j+d}{j}\frac{(j+d)^{j+d-5/2}}{d^{d-3/2}j^{j-3/2}} d^{1-j}\\
    &\sim d.
\end{align*}
Let 
\begin{align*}
    m_j &= \sum_{i=0}^{j}N(d,i)= \sum_{i=0}^j[\frac{\Gamma(i+d)}{\Gamma(d)\Gamma(i+1)} - \frac{\Gamma(i+d-2)}{\Gamma(d)\Gamma(i-1)}\\
    &= \frac{\Gamma(j+d)}{\Gamma(d)\Gamma(j+1)} + \frac{\Gamma(j+d-1)}{\Gamma(d)\Gamma(j)}\sim d^j.
\end{align*}
If inequality \eqref{cod_gaussian} holds, then
\begin{equation*}
    1 \gtrsim d^{-\alpha}m_j^\beta \Lambda_{k,\tau_{d-1}}(m_j) \gtrsim d^{\beta j - \alpha} N(d,j+1)\lambda_{j+1} \sim d^{\beta j - \alpha + 1},
\end{equation*}
which is a contradiction.
\qed
%%%%%%%%%%%%%%%%%%%%%%%%%%%%%%%%%%%%%%%%%%%%%%%%%%%%%%%%%%%%%%%%%%%%%%%%%%%%%%%%%%%%%%%%%%%%%%%%%%%%%%%%%%%%%%%%%%%%%%%%%%%%%%%%%%%%%%%%%%%%%%%%%%%%%%%%%%%%%%%%%%%%%

% \newpage
% \input{appendix_proof}

\end{document}